\documentclass{article}
\usepackage[utf8]{inputenc}
\usepackage{style}

\newcommand{\cD}{\mathcal{D}}
\newcommand{\cF}{\mathcal{F}}

\title{Adversarial Robustness May Be at Odds With Simplicity}
\author{Preetum Nakkiran\thanks{
This work supported in part by
NSF GRFP Grant No. DGE1144152,
and Madhu Sudan’s Simons Investigator Award and NSF Award CCF
1715187.
}
\\Harvard University\\\texttt{preetum@cs.harvard.edu}}
\date{January 2019}

\begin{document}

\maketitle

\begin{abstract}
Current techniques in machine learning are so far are unable to learn classifiers
that are robust to adversarial perturbations.
However, they are able to learn non-robust classifiers with very high accuracy, even in the presence of random perturbations.
Towards explaining this gap,
we highlight the hypothesis that
\emph{robust classification may require more complex classifiers (i.e. more capacity) than standard classification.}

In this note, we show that this hypothesis is indeed possible, by giving several theoretical examples
of classification tasks and sets of ``simple'' classifiers for which:
\begin{enumerate}
    \item There exists a simple classifier with high standard accuracy,\\
    and moreover high accuracy under random $\ell_\infty$ noise.
    \item Any simple classifier is not robust:
    it must have high adversarial loss with $\ell_\infty$ perturbations.
    \item Robust classification is possible, but only with more complex classifiers
    (exponentially more complex, in some examples).
\end{enumerate}
Moreover, \emph{there is a quantitative trade-off between robustness and standard accuracy among simple classifiers.}
This suggests an alternate explanation of this phenomenon, which appears in practice:
the tradeoff may occur not because the classification task inherently requires such a tradeoff (as in [Tsipras-Santurkar-Engstrom-Turner-Madry `18]), but because the structure of our current classifiers imposes such a tradeoff.
\end{abstract}

\section{Introduction}
Current machine-learning models are able to achieve high-accuracy on a variety of classification tasks (such as image recognition),
but it is now well-known that standard models are susceptible to \emph{adversarial examples}: small perturbations of the input which are imperceptible to humans, but cause misclassification \cite{szegedy2013intriguing}.
Various techniques have been proposed to learn models that are robust to small adversarial perturbations, but so far these robust models have failed to be nearly as accurate as their non-robust counterparts \cite{tsipras2018robustness, athalye2018obfuscated}.
Perhaps surprisingly, it is easy in practice to learn classifiers robust to small \emph{random} perturbations, but not to small \emph{adversarial} perturbations.

This motivates at least two questions in the area:

\begin{center}
    \emph{Q1 (informal). Why do current techniques fail to learn good adversarial-robust classifiers,
    while they suffice to learn good standard, and noise-robust classifiers?}
    
    \medskip
    
    \emph{Q2 (informal). Why is there an apparent trade-off between robustness and accuracy for current classifiers?}
\end{center}

To state these questions more precisely, we recall the notions of standard and adversarial loss.
In standard classification we have a data distribution $\cD$
over pairs $(x, y)$, with inputs $x \in \R^d$ and labels $y \in \mathcal{Y}$.
We wish to construct a classifier $f: \R^d \to \mathcal{Y}$
with low \emph{standard loss}\footnote{We consider binary loss here for simplicity.}:
$$\text{(standard loss):}\quad \E_{(x, y) \sim D}[ \1\{f(x) \neq y \}]$$
In the learning setting, we have some family of classifiers $\mathcal{F}$,
and we  wish to find a classifier $f \in \mathcal{F}$ with low standard loss, given independent examples $(x_i, y_i) \sim \cD$ from the distribution.

We may also wish to find classifiers robust against \emph{uniformly random}
$\ell_\infty$ noise,
in which case we want low \emph{noise-robust loss}:
$$\text{(noise-robust loss):}\quad \E_{(x, y) \sim D}[
\E_{\Delta \sim Unif(B_\infty(\eps))}\{f(x + \Delta) \neq y \}]$$

In adversarially-robust classification, we want to protect against an adversary that is allowed to perturb the input, knowing the classifier $f$ and input $x$. Following \cite{madry2017towards}, 
this is modeled as finding a classifier $f$ with low \emph{adversarial loss}:
$$\text{(adversarial loss):}\quad \E_{(x, y) \sim D}[ \max_{\Delta: ||\Delta||_\infty \leq \eps}
\1\{f(x + \Delta) \neq y \}]$$
In this note, we focus on $\ell_\infty$-bounded adversaries,
motivated by the setting of adversarial examples on images.
Now we can state the two questions as:

\begin{center}
    \emph{Q1. Why do current techniques fail to learn classifiers with low adversarial loss,
    while they suffice to learn classifiers with low standard loss, and low noise-robust loss?}
    
    \medskip
    
    \emph{Q2. Why is there a tradeoff between adversarial-loss and standard-loss among current classifiers?}
\end{center}

One possible explanation for the above is that robust classifiers simply do not exist -- that is, the distribution we wish to classify is inherently ``hard'', and does not admit robust classifiers. In this note, we reject this explanation, since humans appear to robustly classify images. Under the assumption that robust classifiers exist, there are several hypotheses for Question 1 in the literature:


\begin{enumerate}[(A)]
    \item The sample-complexity of learning a robust classifier is higher than that of a standard classifier.
    \item The computational-complexity of learning a robust classifier is higher than that of a standard classifier.
    \item The complexity (e.g. capacity) of a robust classifier must be higher than that of a standard classifier.
\end{enumerate}

Note that Hypotheses (A) and (B) are about the difficulty of the \emph{learning problem},
while Hypothesis (C) involves only the classification task.
We show that there exist settings where Hypothesis (C) explains both Questions 1 and 2.
Concretely, this means our failure to train high-accuracy robust classifiers may be because we are not searching within a sufficiently rich class of classifiers, regardless of the sample-complexity or computational-complexity of the learning this classifier.

\begin{remark}
These hypotheses are not necessarily disjoint, and more fine-grained hypotheses are possible.
For example, the hypothesis that ``SGD-based adversarial-training on neural networks fails to learn robust classifiers, even when robust neural networks exist'' could fall under both Hypothesis (B) and (C). In particular, it could be the case that SGD is not a sufficiently powerful learning algorithm, and moreover that networks learnt by SGD are too ``simple'' to be robust.
As a first step, we focus on the coarse-grained hypotheses above.
\end{remark}

{\bf Contributions.}
In this note, we show that there exist classification tasks where Hypothesis (C) is provably true, and explains Questions 1 and 2.
This shows that Hypothesis (C) is not vacuous in a strong sense,
and we hope these theoretical examples can yield insight into the phenomenon in practice.

Specifically, we give several examples of a distribution $(x, y) \sim \cD$ and a family of ``simple'' classifiers $\mathcal{F}$ for which the following properties provably hold:
\begin{enumerate}
    \item There exists a simple classifier $f \in \cF$ with low standard loss, and low noise-robust loss.
    \item Every simple classifier $f \in \cF$ is not adversarially robust; it has high adversarial loss w.r.t $\ell_\infty$ perturbations.
    \item There exists a robust classifier $f^*$ with low adversarial loss (but is not simple).
    \end{enumerate}
The ``simple'' class for us can be taken to be the set of Linear Threshold Functions.
In one of our examples, any robust classifier must take exponential time.
That is, we show
\begin{center}
\emph{robust classification may be exponentially more complex than standard classification.}
\end{center}

Further, the simple classifier that minimizes adversarial-loss has very high standard-loss.
More generally, we show
\begin{center}
\emph{there exists a quantitative tradeoff between robustness and accuracy among ``simple'' classifiers.}
\end{center}
This suggests an alternate explanation of this tradeoff, which appears in practice:
the tradeoff may be happening not because the distribution inherently requires such a tradeoff (as in \cite{tsipras2018robustness}), but because the structure of our current classifiers imposes such a tradeoff.

Our constructions have the additional property that
    the distribution of the input $x \in \R^d$ has all coordinates with the same marginal distribution,
which is inspired by classification tasks in practice (e.g. image classification).

\subsection{Related Work}
We focus only on works directly related to Questions 1 and 2 above.
First, there are several works arguing that robust classifiers simply may not exist \cite{shafahi2018are, tsipras2018robustness, fawzi2018adversarial}.
Notably, Tsipras et al. \cite{tsipras2018robustness} argues that the tradeoff between robustness and accuracy may be an inevitable feature of the classification task -- they give an example of a classification task for which a good robust classifier provably does not exist,
although a good standard classifier does exist.
In contrast, we reject this explanation, since we are working under the assumption that robust classifiers do exist (e.g. humans).

Madry et al. \cite{madry2017towards} acknowledges Hypothesis (C), and gives empirical evidence towards this by showing that increased network capacity appears to help with adversarial robustness, up to a point.
Schmidt et al. \cite{schmidt2018adversarially} proposed Hypothesis (A), observing that adversarial-loss has larger generalization error than standard-loss in practice.
They further give a theoretical example where learning a robust classifier requires polynomially more samples than learning a standard classifier.
Bubeck et al. \cite{bubeck2018adversarial} shows that this polynomial gap in sample-complexity is the worst possible gap under reasonable assumptions -- that is, it is often information-theoretically possible to learn a robust classifier if one exists, from only polynomially-many samples.
Further, Bubeck et al. \cite{bubeck2018adversarial, bubeck2018adversarial-crypto} propose Hypothesis (B), and give a theoretical example of a learning task where learning a robust classifier is not possible in polynomial time (under standard cryptographic assumptions).

In contrast, we focus on Hypothesis (C), which involves only the classification task and not the learning task.
As far as we are aware, we give the first theoretical examples demonstrating this hypothesis, and showing that the tradeoff between standard and adversarial loss may exist only among restricted sets of classifiers.

\subsection{Overview of Our Constructions.}
{\bf Construction 1.}
Consider the following classification task.
Let the distribution $\cD$ over pairs $(x, y)$ be defined as:
Sample $y \sim \{+1, -1\}$ uniformly, and sample each coordinate of $x \in  \R^n$
independently as
$$
x_i =
\begin{cases}
+y &\text{w.p. 0.51}\\
-y &\text{w.p. 0.49}
\end{cases}
$$
Recall, we wish to predict the class $y$ from the input $x$.
Let the set of ``simple'' classifiers $\cF$ be Linear Threshold Functions, of the form
$f_w(x) = \1\{\langle w, x \rangle > 0\}$ for $w \in \R^n$.
Consider $\ell_\infty$ adversarial perturbation of up to $\eps = 
\frac{1}{2}$.
Note that:
\begin{enumerate}
    \item There exists a linear classifier with
    $\text{(standard-loss)} \leq \exp(-\Omega(n))$.\\
    For example,
    $f_{\1}(x) := \1\{\sum^n_{i=1} x_i > 0\}$.
    \item Every linear classifier has $\text{(adversarial-loss)} \geq \Omega(1)$.\\
    The sum $\sum_i x_i$ above, for example, concentrates around $\pm 0.01n$ but can be perturbed by $\eps n = n/2$ by an $\ell_\infty$ adversary, causing mis-classification with high probability.
    \item There exist classifiers with $\text{(adversarial-loss)} \leq \exp(-\Omega(n))$.\\
    For example, $f^*(x) := \1\{ \sum^n_{i=1} \text{Round}(x_i) > 0 \}$ where
    $\text{Round}(\cdot)$ rounds its input to $\{-1, 1\}$.
\end{enumerate}
Moreover, the linear classifier of (1) is robust to random $\ell_\infty$ \emph{noise} of order $\eps$, but just not to adversarial perturbation.

We also have a tradeoff between adversarial-loss and standard-loss in this setting:
Linear classifiers with larger support will be more accurate on the standard distribution, but their larger support makes them more vulnerable to adversarial perturbation.
Quantitatively, consider for simplicity the subclass of linear threshold functions
$\cF' = \{f_w : w \in \{0, 1\}^n \}$.
Then there exists some universal constant $\gamma$
such that
(Theorem~\ref{thm:loss-tradeoff}):
$$
\boxed{
\forall f \in \cF': \quad \text{(adversarial-loss of $f$)} + \text{(standard-loss of $f$)}^\gamma \geq 1
}
$$

The above example may be unsatisfying, since the more ``complex'' classifier simply
pre-processes its input by rounding.
This is specific to the binary setting, where we can always eliminate the effect of
any perturbation by rounding to $\{\pm 1\}$.
One may wonder if it is always possible to construct a robust classifier by ``pre-processing'' a simple standard classifier, and if robust classifiers are always just slightly more complex than standard ones.
The next construction shows that this is not the case.

{\bf Construction 2.}
The following construction shows that robust classification can require exponentially more complex classifiers than simple classification.
Let $g: \{0, 1\}^n \to \{0, 1\}$ be a function that is \emph{average-case hard},
such that any $2^{O(n)}$-time nonuniform algorithm cannot compute $z \mapsto g(z)$ noticeably better than random guessing.
For example, taking $g$ to be a random function from $\{0, 1\}^n \to \{0, 1\}$ suffices.
For any $\eps < 1$, let the distribution be defined over $(x, y)$ as follows.
Sample $z \in \{0, 1\}^n$ uniformly, and let $x = (\eps g(z), z) \in \R^{n+1}$.
Let $y = g(z)$.
Recall, we wish to predict $y$ from the input $x$. In this setting:
\begin{enumerate}
    \item The dictator function $f(x) := x_1$ has $\text{(standard-loss)} = 0$.
    \item Any classifier running in time $\leq 2^{O(n)}$ has
    $\text{(adversarial-loss)} \geq \frac{1}{2} - \exp(-\Omega(n))$.
    \item There exists a classifier with $\text{(adversarial-loss)} = 0$.
\end{enumerate}
Here, an adversarial perturbation of order $\eps$ can destroy the first coordinate of $x$, and thus robustly predicting $y$ from $x$ requires actually computing the function $g(z)$ -- which we assume requires time $2^{\Omega(n)}$.

This example can be extended to have all coordinates marginally uniform over $[0, 1]$,
so the first coordinate is not distinguished in this way.

\begin{remark}
We took $g$ to be average-case hard for $2^{\Omega(n)}$-time,
which is a distribution unlikely to appear in Nature.
For a more ``realistic'' example, we can take $g$ to be hard with respect to the class of classifiers currently under consideration. For example, $g$ could be a function that is average-case hard for neural-networks of a specific bounded size.
\end{remark}

\begin{remark}
In this example, there is an ``easy but fragile'' feature that a standard classifier can use, but which can be destroyed by adversarial perturbation.
A robust classifier, however, cannot ``cheat'' using this feature, and has to
in some sense actually solve the problem.
\end{remark}

\section{Formal Constructions}
\newcommand{\StdLoss}[2]{\mathrm{StdLoss}_{#1}(#2)}
\newcommand{\AdvLoss}[2]{\mathrm{AdvLoss}_{#1, \eps}(#2)}
\newcommand{\NoisyLoss}[2]{\mathrm{NoisyLoss}_{#1, \eps}(#2)}
Here we formally define and state our claimed properties about Constructions 1 and 2; proofs appear in the Appendix.

\begin{definition}[Loss Functionals]
For any distribution $\cD$ over $\R^n \x \mathcal{Y}$,
$\eps > 0$, and any function $f: \R^n \to \mathcal{Y}$,
define the standard loss of $f$ as:
$$\StdLoss{D}{f} := \E_{(x, y) \sim D}[ \1\{f(x) \neq y\}]$$
the adversarial loss of $f$ as:
$$
\AdvLoss{D}{f} :=
\E_{(x, y) \sim D}\left[ \max_{\Delta \in \R^n: ||\Delta||_\infty \leq \eps} 
\1\{f(x + \Delta) \neq y\} \right]
$$
and the noise-robust loss of $f$ as:
$$\NoisyLoss{D}{f} :=
\E_{(x, y) \sim D}\left[ \E_{\Delta \sim
\mathrm{Unif}(\{\delta \in \R^n: ||\delta||_\infty \leq \eps\} )}
\1\{f(x + \Delta) \neq y\} \right]$$
\end{definition}

\subsection{Construction 1}

\begin{definition}[Construction 1]
Define $\cD_1$ as the following distribution over $(x, y)$.
Sample $y \sim \{+1, -1\}$ uniformly, and sample each coordinate of $x \in  \R^n$
independently as
$$
x_i =
\begin{cases}
+y &\text{w.p. 0.51}\\
-y &\text{w.p. 0.49}
\end{cases}
$$
\end{definition}

Let $\cF = \{f_w(x) := \mathrm{sign}(\langle w, x \rangle) : w \in \R^n\}$ be the set of 
linear classifiers.

\begin{theorem}
\label{thm:c1}
For all $\eps \in (0.01, 1)$, the distribution $D_1$ of Construction 1 satisfies the following properties.
\begin{enumerate}
    \item There exists a linear classifier $f \in \cF$ with 
    $\StdLoss{D_1}{f} \leq \exp(-\Omega(n))$ and 
    $\NoisyLoss{D_1}{f} \leq \exp(-\Omega(n))$.
    \item Every linear classifier $f \in \cF$
    has $\AdvLoss{D_1}{f} \geq \Omega_\eps(1)$.
    \item There exists a (non-linear) classifier $f^*: \R^n \to \{\pm 1\}$ with
    $\AdvLoss{D_1}{f^*} \leq \exp(-\Omega(n))$.
\end{enumerate}
Where the $\Omega(\cdot)$ hides only universal constants, and $\Omega_\eps(\cdot)$ hides constants depending only on $\eps$.
\end{theorem}

\begin{theorem}[Loss Tradeoff]
\label{thm:loss-tradeoff}
Consider the subset of linear classifiers
$\cF' = \{f_w(x) = \mathrm{sign}(\langle w, x \rangle) : w \in \{0, 1\}^n\}$.
For all $\eps \in (0.01, 1)$, there exists a constant $\gamma$ such that for all $n$,
$$
\forall f \in \cF': \quad \AdvLoss{D_1}{f} + \left(\StdLoss{D_1}{f}\right)^\gamma \geq 1
$$
\end{theorem}

\subsection{Construction 2}
We first need the notion of an average-case hard function.
\begin{definition}[Average-Case Hard]
A boolean function $g: \{0, 1\}^n \to \{0, 1\}$ is
\emph{$(s, \delta)$-average-case hard}
if for all non-uniform probabilistic algorithms $A$ running in time $s$,
$$\Pr_{A, x\in \{0, 1\}^n}[A(x) \neq g(x)] \geq \delta$$
\end{definition}

There exists functions $g$ which are 
$(2^{O(n)}, 1/2 - 2^{-\Omega(n)})$-average-case hard
(a random function $g$ will suffice with constant probability).
We now define Construction 2; note that this extends the presentation in the Introduction
by having all coordinates marginally uniform, and also admitting a simple noise-robust classifier.
\begin{definition}[Construction 2]
For a given function $g: \{0, 1\}^n \to \{0, 1\}$,
and a given $\eps$,
define $\cD_{g, \eps}$ as the following distribution over $(x, y)$.

Sample $z \in \{0, 1\}^n$ uniformly at random.
Sample $a, b \in [0, 1]^n$ with each coordinate independently uniform $a_i, b_i \in [0, 1]$.
Define $x = (\alpha, \beta) \in [0, 1]^{4n}$ and $y \in \{0, 1\}$ as:
\begin{align*}
x &:= \{(a_i, a_i + 2\eps g(z) ~\mathrm{mod}[0, 1])\}_{i \in [n]}
\circ
\{ (b_i, b_i + 0.5z_i ~\mathrm{mod}[0, 1]) \}_{i \in [n]}\\
y &:= g(z)
\end{align*}
where $\circ$ denotes concatenation.
\end{definition}

\begin{theorem}
\label{thm:c2}
For all functions $g: \{0, 1\}^n \to \{0, 1\}$ that are $(s(n), \delta(n))$-average-case hard, and
all $\eps \in (0, 1/8)$,
the distribution $D_{g, \eps}$ of Construction 2 satisfies the following properties.
\begin{enumerate}
    \item
    The following classifier
    $$f(\alpha, \beta) =
    \1\{\exists i \in [n]: (\alpha_i - \alpha_{i+1}) ~\mathrm{mod}[0, 1] \geq 2\eps \} $$
    has
    $\StdLoss{D}{f} = 0$ and 
    $\NoisyLoss{D}{f} \leq \exp(-\Omega(n))$.
    
    \item Every classifier algorithm $f$ running in time $s(n) - \Theta(n)$
    has $\AdvLoss{D}{f} \geq \delta$.
    
    \item There exists a classifier $f^*: \R^n \to \{\pm 1\}$ with
    $\AdvLoss{D}{f^*} = 0$.
\end{enumerate}
In particular, we can take $g$ to be $(s(n) = 2^{O(n)}, \delta(n) = 1/2 - 2^{-\Omega(n)})$ average-case hard.
\end{theorem}

\section{Acknowledgements}
The author thanks Ilya Sutskever for asking the question that motivated this work.
We also thank Kelly W. Zhang for helpful discussions, and Ben Edelman for comments on an early draft.

\bibliographystyle{plain}
\bibliography{refs}

\appendix
\section{Proofs}
\begin{proof}[Proof of Theorem~\ref{thm:c1}]
For property (1), consider the linear classifier
$f_{1}(x) := \textrm{sign}(\sum_{i = 1}^n x_i)$. The upper-bounds on $\StdLoss{}{}$ and $\NoisyLoss{}{}$ follow from Chernoff-Hoeffding bounds:
\begin{align*}
\StdLoss{D_1}{f_1}
=
\Pr_{(x, y) \sim D}[ \textrm{sign}(\sum_i x_i) \neq y]
=
\Pr_{(x, y) \sim D}[ \sum_i yx_i  < 0]
\leq \exp(-\Omega(n))
\end{align*}
since $\E[yx_i] = +0.01$, and $\{x_i\}_{i \in [n]}$ are conditionally independent given $y$. 
Similarly,
\begin{align*}
\NoisyLoss{D_1}{f_1}
&=
\Pr_{(x, y) \sim D, \eta_i \sim [-\eps, +\eps]}[ \textrm{sign}(\sum_i x_i + \eta_i) \neq y]\\
&=
\Pr_{(x, y) \sim D, \eta_i \sim [-\eps, +\eps]}[ \sum_i y(x_i + \eta_i) < 0]\\
&\leq \exp(-\Omega(n))
\end{align*}

For property (2), this follows from Azuma-Hoeffding.
For any linear classifier $f_w$, the adversarial success is bounded by:
\begin{align*}
1-\AdvLoss{D_1}{f_w}
&= 1-\E_{(x, y) \sim D}\left[
\max_{\Delta \in \R^n: ||\Delta||_\infty \leq \eps} 
\1\{f_w(x + \Delta) \neq y\} \right]\\
&= \E_{(x, y) \sim D}\left[
\min_{\Delta \in \R^n: ||\Delta||_\infty \leq \eps} 
\1\{f_w(x + \Delta) = y\} \right]\\
&= \Pr_{(x, y) \sim D}\left[
\min_{\Delta \in \R^n: ||\Delta||_\infty \leq \eps} 
y\langle w, x + \Delta \rangle
> 0 \right]\\
&= \Pr_{(x, y) \sim D}\left[
y\langle w, x \rangle
- \eps ||w||_1
> 0 \right] \tag{minimized at $w = \eps y ~\mathrm{sign}(w)$}\\
&= \Pr_{z_i \sim D'}\left[
\langle w, z \rangle > \eps ||w||_1
\right] \tag{for $D' = 2B(0.51) - 1$, the distribution of $yx_i$}\\
&\leq 
\exp\left( -\frac{\delta^2 ||w||_1^2}{2||w||_2^2} \right)
\tag{$\star$. For $\delta := \eps - \E[z_i] > 0$, by Azuma-Hoeffding.}\\
&\leq 
\exp\left( -\frac{\delta^2}{2} \right)
= \exp\left( -\frac{(\eps - 0.01)^2}{2} \right)
= 1- \Omega_\eps(1)
\end{align*}
In line $(\star)$, note
$\E_{z \sim D'}[\langle w, z \rangle] = \E[z_1] \sum_i w_i < \E[z_i] ||w||_1$,
and in the next line $||w||_1 \geq ||w||_2$.

Thus, for any linear classifier $f_w$ we have
$$\AdvLoss{D_1}{f_w} \geq \Omega_\eps(1)$$

For property (3), simply consider the classifier
$f^*(x) :=  \mathrm{sign}( \sum^n_{i=1} \mathrm{Round}(x_i) )$ where $\mathrm{Round}(\cdot)$
rounds its argument to $\{\pm 1\}$.
Clearly, $\AdvLoss{D_1}{f^*} = \StdLoss{D_1}{f_1}$ since the rounding inverts the effect of any perturbation (for $\eps < 1$).
Thus, the adversarial loss is $\exp(-\Omega(n))$, as in the stadard classifier of property (1).

\end{proof}

\begin{proof}[Proof of Theorem~\ref{thm:loss-tradeoff}]
For a linear classifier $f_w : w \in \{0, 1\}$, let $k = \mathrm{supp}(w)$ be its support.
By Azuma-Hoeffding as in the proof of Theorem~\ref{thm:c1}, the adversarial-success is upper-bounded by:
$$
1-\AdvLoss{D_1}{f_w}
\leq 
\exp\left( -\frac{\delta^2 ||w||_1^2}{2||w||_2^2} \right)
= \exp(-\delta^2 k / 2)
$$
where $\delta = (\eps-0.01)^2 \geq \Omega(1)$.
Now, the standard loss is lower-bounded by:
$$
\StdLoss{D_1}{f_w}
=
\Pr_{(x, y) \sim D}[ \textrm{sign}(\sum_{i \in \mathrm{supp}(w)} x_i) \neq y]
\geq
\Pr_{(x, y) \sim D}[ \forall i \in \mathrm{supp}(w): x_i \neq y]
\geq (0.49)^k
$$
Combining these two bounds yields
$$
1-\AdvLoss{D_1}{f_w}
\leq 
\exp(-\delta^2 k / 2)
=
\exp(-k \ln(1/0.49))^\gamma
\leq
(\StdLoss{D_1}{f_w})^\gamma
$$
for
$\gamma :=
\frac{\delta^2}{2\ln (1/0.49)}
=\frac{(\eps - 0.01)^2}{2\ln (1/0.49)}
$.
\end{proof}

\begin{proof}[Proof of Theorem~\ref{thm:c2}]
Write the input $x$ as $x = (\alpha, \beta)$ for $\alpha, \beta \in \R^{2n}$.

For property (1): The bound on standard loss follows directly from the encoding.
The bound on noisy loss follows because, for every $i \in [n]$,
the event $\{(\alpha_i - \alpha_{i+1})~\mathrm{mod} [0, 1] \geq 2\eps \}$
occurs with probability $0$ if $g(z) = 0$, and with probability at least $\Omega(1)$ if $g(z) = 1$.

For property (2): First, consider the $\eps$-bounded adversary
that adds $(+\eps g(z), - \eps g(z), +\eps g(z), -\eps g(z), \dots)$ to the input's $\alpha$. Notice this adversary perturbs the input $x = (\alpha, \beta)$ such that $\alpha$ is independent of $z$ in the perturbed distribution.
Now, suppose for the sake of contradiction that there existed a classifier running in time $s(n) - \Theta(n)$, with loss less than $\delta$ on this perturbed distribution.
This would yield a time-$s(n)$ algorithm for computing $z \mapsto g(z)$ with error better than $\delta$: simply simulate the (perturbed) inputs to the classifier, which can be done in time $O(n)$, and output the result of the classifier.
Thus, such a classifier cannot exist, since $g$ is average-case hard.

For property (3): Notice that $z$ can be easily decoded from $\beta$, even with perturbations of up to $\eps < 1/4$. Thus, the classifier which decodes $z$, then computes $g(z)$ has $0$ adversarial loss.
\end{proof}

\end{document}